\pdfoutput=1
\documentclass[11pt]{article}
\usepackage{ACL2023}
\usepackage[T1]{fontenc}
\usepackage[utf8]{inputenc}

\usepackage{amsfonts}
\usepackage{amsmath}
\usepackage{amsthm}
\usepackage{xspace}
\usepackage[inline]{enumitem}
\usepackage{graphicx}     
\usepackage{url}          
\usepackage{xcolor}
\usepackage[capitalise]{cleveref}

\newtheorem{theorem}{Theorem}
\newtheorem{lemma}{Lemma}
\theoremstyle{definition}
\newtheorem{definition}{Definition}

\newcommand{\len}[1]{\lVert#1\rVert}
\newcommand{\trunc}[1]{\left[#1\right]}

\newcommand{\head}[1]{\mathit{Att}_{#1}}
\newcommand{\scoring}{\ensuremath{\sigma}\xspace}
\newcommand{\N}{\mathbb{N}}
\newcommand{\Z}{\mathbb{Z}}
\newcommand{\Q}{\mathbb{Q}}
\newcommand{\R}{\mathbb{R}}
\newcommand{\transpose}{^{\mathrm T}}

\newcommand{\F}{\mathbb{F}}
\newcommand{\TC}{\mathsf{TC}}
\newcommand{\AC}{\mathsf{AC}}

\newcommand{\mostsim}{\mathcal{M}}

\newcommand{\igate}{\mathsf{in}}

\newcommand{\fmax}{\mathsf{max}}
\newcommand{\feq}{\mathsf{eq}}

\newcommand{\fsel}{\mathsf{sel}}
\newcommand{\fsum}{\mathsf{sum}}

\newcommand{\fdiv}{\mathsf{div}}

\newcommand{\enc}{\mathsf{enc}}

\newcommand{\true}{1}
\newcommand{\false}{0}

\newcommand{\absatz}{\par\vspace{1em}\noindent}

\usepackage{times}
\usepackage{latexsym}

\usepackage{microtype}
\usepackage{inconsolata}
\usepackage{csquotes}

\title{Average-Hard Attention Transformers are Constant-Depth Uniform Threshold Circuits}

\author{Lena Strobl \\
  Department of Computing Science \\
  Ume{\aa} University, Sweden \\
  \href{mailto:lena.strobl@umu.se}{lena.strobl@umu.se} \\}

\begin{document}

\maketitle

\begin{abstract}
    Transformers have emerged as a widely used neural network model for various natural language processing tasks.
    Previous research explored their relationship with constant-depth threshold circuits, making two assumptions: average-hard attention and logarithmic precision for internal computations relative to input length.
    \citet{MerrillSabharwalSmith2021} prove that average-hard attention transformers recognize languages that fall within the complexity class $\TC^0$, denoting the set of languages that can be recognized by constant-depth polynomial-size threshold circuits.
    Likewise, \citet{MerrillSabharwal2023} show that log-precision transformers recognize languages within the class of \emph{uniform} $\TC^0$.
    This shows that both transformer models can be simulated by constant-depth threshold circuits, with the latter being more robust due to generating a uniform circuit family.
    This paper shows that the first result can be extended to yield uniform circuits as well.
\end{abstract}
\section{Introduction}

The dominance of recurrent neural network (RNN) architectures in the realm of natural language processing gradually waned with the advent of transformers, as initially introduced by \citet{vaswani2017attention}.
Unlike RNNs, which heavily rely on autoregressive mechanisms, transformers revolutionized the field by leveraging parallelism to process sequential data.

While RNNs could be analyzed through the lens of automata theory (notably by \citet{weiss2018practical, peng2018rational}) thanks to their recurrence-based nature, the characterization of transformers necessitates a different approach.
Considering the circuit-based perspective seems natural, given the absence of explicit recurrence in transformers.
Notably, some studies have attempted to reintroduce recurrences into transformers, as exemplified by the work on shortcut connections by \citet{liu2023transformers}.
However, for coherence and maintaining the focus of our discussion, we will refrain from delving deeper into these tangential directions.

Recent advances in the analysis of transformer models have shed light on their computational capabilities, particularly through the investigation of two distinct formal models: average-hard (= saturated) by \citet{MerrillSabharwalSmith2021} and softmax (= soft) transformers by \citet{MerrillSabharwal2023}.
Average-hard attention enables a connection to be established between these models and devices of formal language theory.
Building upon this line of research, \citet{MerrillSabharwal2023} introduced a different model, demonstrating that transformer networks with logarithmic precision in relation to the input length can be simulated by constant-depth uniform threshold circuits.
Consequently, the complexity class $\TC^0$ serves as an upper bound for the formal languages recognized by these transformers.

Motivated by the inherent uniformity possessed by transformers, we want to investigate whether average-hard attention transformers only recognize languages in uniform $\TC^0$.
Our primary contribution lies in our proof, showcasing that average-hard attention transformers can indeed be simulated by uniform  $\TC^0$ circuits, thereby solidifying their association with uniform $\TC^0$.
Consequently, these transformers are inherently limited to solving problems within uniform $\TC^0$.

This result does not follow from the result presented by \citet{MerrillSabharwal2023}, as both the underlying assumptions and the specific attention mechanisms differ between the two studies.
Concretely, we consider the implications of the results from \citet{MerrillSabharwal2023} and \citet{MerrillSabharwalSmith2021}:
\citet[Theorem 4]{MerrillSabharwalSmith2021} demonstrated that average-hard attention transformers are only capable of producing floating-point numbers of logarithmic size.
Consequently, one might argue that average-hard attention transformers can be considered log-precision transformers, and therefore the result \citet{MerrillSabharwal2023} establish should be applicable in this context.
\citet{MerrillSabharwalSmith2021} Theorem 4 relies on the assumption of \enquote{size preserving} functions, while we adopt the fundamental definitions provided by \citet{MerrillSabharwal2023}.
This discrepancy in the underlying assumptions creates a distinction between the two frameworks.
Furthermore, it should be emphasized that the attention mechanism \citet{MerrillSabharwal2023} employed is softmax, which is a difference in the formal definition of attention itself.
As a result, even disregarding the question of precision, the direct applicability of \citet{MerrillSabharwal2023} result would be incorrect as it would disregard the difference in the attention mechanisms.

\absatz
The findings of this paper open up new approaches for future research, specifically in exploring the distinction between average-hard and softmax attention mechanisms, with the potential to unveil a clear demarcation between the two.

\begin{table}[!ht]
\centering
\includegraphics[width=\columnwidth]{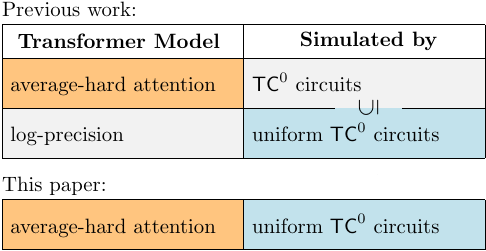}
\caption{Overview of previous results by \citet{MerrillSabharwalSmith2021,MerrillSabharwal2023} (top) and the  contribution of this paper (bottom).
Previous research demonstrated the simulation of average-hard attention transformers using $\TC^0$ circuits with integer values.
Another study showed the simulation of log-precision transformers with uniform $\TC^0$ circuits.
In this paper, we extend these results by demonstrating the simulation of average-hard attention transformers with uniform $\TC^0$ circuits.}
\label{table:1}
\end{table}

\section{Preliminaries}

In this section, we establish the foundational definitions and notation for circuit computations, drawing from the textbook of \citet{arora2006computational} in Chapters 6 and 14.
This established framework forms the basis for our subsequent analysis.

Moreover, we revisit the average-hard attention transformer model proposed by \citet{MerrillSabharwal2023}.
To do so, we provide definitions of average-hard attention and average-hard attention heads.
These essential concepts serve as the cornerstone of the average-hard attention transformer model.

\subsection{Basic Mathematical Notation and Definitions}

We employ the notation and definitions commonly used in mathematics and formal language theory.
Specifically, we denote the sets of natural numbers, including zero, and integers as $\N$ and $\Z$ respectively.

\absatz
For any natural number $n$, the set containing the numbers from $1$ to $n$ (inclusive) is denoted as $[n]$.
Notably, when $n = 0$, we represent the set as $\emptyset$.

\absatz
The set of all strings composed of elements from a given set $\Sigma$ is represented as $\Sigma^*$.
Here, we denote the empty string as $\epsilon$, and we define $\Sigma^+=\Sigma^*\setminus\{\epsilon\}$.

\absatz
The canonical extension of a function $f\colon\Sigma\to\Delta$ to a function from $\Sigma^*$ to $\Delta^*$ is denoted by $f$ as well.
Thus, $f(\sigma_1\cdots\sigma_n)=f(\sigma_1)\cdots f(\sigma_n)$ for all $\sigma_1,\dots,\sigma_n\in\Sigma$.
This notation allows us to apply the function $f$ to each individual element within the string.

\absatz
Due to the inherent limitations of Boolean circuits, which can only process values of $\true$ and $\false$, representing floating point numbers used in neural networks becomes a challenge.
To accommodate this discrepancy, these numerical values are transformed into bit strings, belonging to the set $\{0,1\}^*$.
Furthermore, the operations performed on these bit strings must be simulated through Boolean operations, which are the fundamental building blocks available to the specific circuit type under consideration.
Consequently, any manipulations or computations on these floating point representations necessitate a translation into operations that can be expressed using the available Boolean operations.

\paragraph{Binary representation.}
The binary representation of $n\in\Z$ is the unique string \[w=b_0b_1\cdots b_m\in\{0,1\}^+\] with $b_1 = 1$ if $m>0$, and $n = -1 ^ {b_0}\sum_{i=1}^mb_i2^{i-1}$.
We denote the length $m$ of this representation by $\len n$, i.e., \[ \len n=\lceil\log_2(|n|+1)\rceil+1 .\]

\paragraph{Precision.}
Let $p \in \N$ be called \emph{precision}.
Following \citet{MerrillSabharwal2023} work, we define the set $\F_p$ to be the set of all rational numbers that can be written as $m\cdot 2^z$ where $m,z\in\Z$ are such that $\len m,\len z\le p/2$.
(Thus, we may always assume that $p$ is positive and even because $\F_0= \emptyset$ and for odd $p$, $\F_p = \F_{p-1}$.)

In other words, a number in $\F_p$ can be specified by two bit strings of length $p/2$ denoting the mantissa $m$ and the exponent $z$.

\paragraph{Arithmetic on floats.}
Float arithmetic involves performing operations on floating-point numbers by first carrying out computations in $\Q$ and then managing potential overflow and excess bits.

\noindent
To formalize this process, we introduce a value $q$, defined as $2^{\lfloor p/2\rfloor-1}-1$, where $p$ represents the precision.
This value represents the largest natural number $q$ such that $||q|| \leq p/2$.

\noindent
Given a rational number $r \in \mathbb \Q$, and let $\trunc r_p$ denote the truncation of $r$ to a float in $\F_p$, assuming $r > 0$.
This truncation is defined as follows.
To determine the exponent $z$, we select a value within the range $-q$ to $q$ such that multiplying~$r$ by $2^z$ scales it as much as possible without exceeding~$q$.
Next, we truncate the mantissa, retaining up to $\lfloor p/2\rfloor-1$ bits (unless $z=q$, indicating that the exponent would result in an overflow).

\noindent
Formally,
\[\trunc r_p = \begin{cases}
  -\trunc{-r}_p  & \text{if } r<0, \\
  q\cdot2^q  & \text{if } r>q\cdot2^q,\text{ and} \\
    \lfloor q\cdot 2^{z}\rfloor\cdot 2^{-z} & \text{if $0\le r\le q\cdot2^q$,}
\end{cases}\]
where $z$ is the largest integer such that $-q\le z\le q$ and $r\cdot 2^{z}\le q$.

Note, that the choice of $z$ in the third case ensures that we retain the maximum number of bits in the mantissa during the truncation process.

For instance, we have $p=6$ and $r=0.0101_2$, then we select $z=q=3$.
Consequently, the truncation of $r$ to $p$ bits, denoted as $\trunc r_p$, is given by $\lfloor 10.1_2 \rfloor \cdot 2^{-3} = 0.01_2$.

\subsection{Circuit computations}

\paragraph{Circuits.}
A \emph{Boolean circuit}, denoted as $C$, is a directed acyclic procedural computational graph that encompasses binary input gates, represented as $\igate_1,\dots,\igate_n$, which serve as the leaf nodes of the graph.
These input gates correspond to $n$ input values, each taking the value of either $\true$ or $\false$.
Intermediate nodes within the circuit, referred to as internal gates, are composed of basic Boolean functions such as logical OR ($\lor$), logical AND ($\land$), and logical NOT ($\neg$).

\noindent
The output of the circuit, denoted as $C(x)$, is determined recursively by applying the logical operations from the input gates through the graph until reaching the root.
Thus, a Boolean circuit defines a function mapping inputs from $\{0,1\}^k$ to outputs in $\{0,1\}$.
It is also possible to consider circuits with a multiple output gates, allowing the computation of functions from $\{0,1\}^k$ to $\{0,1\}^\ell$, where $\ell$ represents the number of output gates.

\paragraph{Size and depth}
The \emph{size} of a circuit, denoted as $|C|$, is the number of nodes present within the graph.
Additionally, we define the \emph{depth} of the circuit as the longest directed path within it, capturing the length of the computational flow from the input gates to the output.

\absatz
In \cref{fig:XOR-circuit}, a circuit is depicted that performs the logical operation $\mathsf{XOR}$ by receiving two inputs and producing an output of 1 if exactly one of the inputs is 1; otherwise, it outputs 0.

\begin{figure}[!ht]
    \centering
    \includegraphics[]{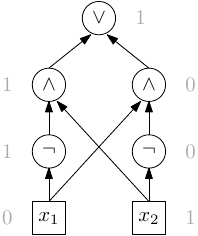}
    \caption{A circuit of size 7 and depth 3 performing the logical $\mathsf{XOR}$ operation on a 2-bit input.}
    \label{fig:XOR-circuit}
\end{figure}

\paragraph{Circuit Families}
In traditional circuit theory, circuits are limited to operating on a fixed input size.
However, we need a model that can handle inputs of arbitrarily long strings as input.
As is customary, we thus use a \emph{circuit family}: a collection $(C_n)_{n \in \mathbb{N}}$, where each circuit $C_n$ has $n$ inputs gates.
Consequently, the size and depth of circuits within this family become functions of $n$, allowing for flexibility in handling inputs of varying lengths.

\absatz
We now recall the definitions of two fundamental classes of circuit families.

\begin{definition}[The Class $\AC^0$]
A language $L$ belongs to the class $\AC^0$ if it can be decided by a circuit family $(C_n)_{n\in\mathbb{N}}$, where each circuit $C_n$ is constructed only using gates from the set $\{\land, \lor, \neg\}$. Furthermore, the circuits in $(C_n)_{n\in\mathbb{N}}$ are required to have polynomial size and constant depth in $n$.
\end{definition}

\begin{definition}[The Class $\TC^0$]
Let $\mathsf{majority}$ be the function that takes a sequence of bits as input and returns 1 if the number of 1s in the sequence is greater than the number of 0s, and 0 otherwise.
A language $L$ is in $\TC^0$ if it can be decided by a circuit family $(C_n)_{n\in\mathbb{N}}$, where each circuit $C_n$ is constructed using gates from the set $\{\land, \lor, \neg, \mathsf{majority}\}$.
Similar to $\AC^0$, the circuits in $(C_n)_{n\in\mathbb{N}}$ have polynomial size and constant depth in $n$.
\end{definition}

\paragraph{Uniform Circuit Families}
A family of circuits $(C_n)_{n\in\mathbb{N}}$ is called \emph{logspace uniform}, or simply \emph{uniform}, if there exists a Turing machine (TM) that can compute $C_n$ from $1^n$ (the number $n$ in unary notation) using $O(\log n)$ space.
In particular, uniform $\TC^0$ is the set of languages that can be decided by a uniform $\TC^0$ circuit family.

\noindent
In the context of this paper, the transformers under study operate on vectors over $\mathbb{F}_p$.
In circuit representations, we typically assume that elements of $\mathbb{F}_p$ are encoded as bit strings of length $p$, obtained by concatenating the mantissa and exponent in binary notation.
To ensure consistency, these bit strings are padded with leading zeroes after the sign bit, making both components exactly $p/2$ in length.

\noindent
A mapping $f\colon (\F_p^k)^n\to(\F_p^k)^n$ is considered \emph{uniformly $\AC_0$ computable} (or \emph{uniformly $\TC_0$ computable}) if there exists a uniform $\AC_0$ (or $\TC_0$) circuit family that, given the bit string representation of $x_1\cdots x_n\in (\F_p^k)^n$ as input, computes $f(x_1,\dots,x_n)$.

\subsection{Transformers} \label{sec:transformers}

A transformer model is composed of a finite number of layers, where each layer comprises multiple so-called attention heads working in parallel, followed by a feed-forward network.
\cref{fig:layer} provides a visual representation of the layer's structure, the arrangement of an individual attention head within the layer as well as the internal configuration of an attention head can be observed in \cref{fig:transformer}.

\begin{figure}
\centering
\includegraphics[width=\columnwidth]{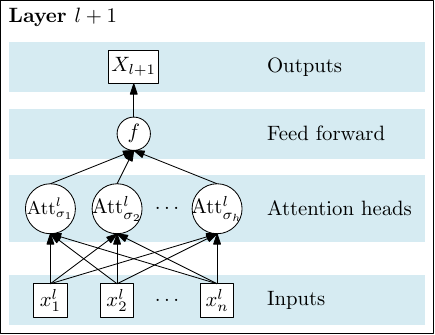}
\caption{Schematic representation showcasing the structure of a transformer layer $l+1$ with the outputs of layer $l$ as inputs, namely $X_l = (x_1^l, \ldots, x_n^l)$, attention heads Att$_{\sigma_1}^l,\ldots$ Att$_{\sigma_h}^l$ and them being combined with a feed forward network $f$ to produce the output $X_{l+1}$.\label{fig:layer}}
\end{figure}

\noindent
In this paper, particular emphasis is placed on the attention head component.
The attention head implements the attention mechanism, which facilitates the mapping of a sequence of $n$ vectors to a probability distribution over the set $[n]$, ultimately yielding a weighted sum of these vectors.

\noindent
The findings presented in this paper are based on the assumption that the vector components are floats with a precision of $O(\log n)$, where $n$ denotes the length of the input.
Consequently, the analysis is focused on transformers in which all internal computations occur within $\mathbb{F}_p$, where the value of $p$ is determined by $p=c_1\log n+c_0$, with $c_0$ and $c_1$ as constants greater than zero.
Throughout the remainder of this paper, the symbol $p$ represents this specific value, which depends on the input length.

\noindent
\begin{quote}
In the subsequent sections of this paper, we will consider a fixed natural number, denoted as $k$, which corresponds to the number of dimensions of the vectors handled by the transformer under consideration.
\end{quote}

\subsection{Attention}

Attention within a transformer model is computed using attention heads.
In this study, our analysis focuses on average-hard attention, as introduced by \citet{MerrillSabharwalSmith2021}.
We proceed by formalizing the concepts of average-hard attention and average-hard attention heads.

\begin{definition}[Average-hard attention function]\label{def:average-hard-attention-score}
    For $s=s_1\cdots s_n\in\F_p^+$, let $\mostsim(s)=\{i\in[n]\mid s_i=\max_{j\in[n]}s_j\}$.
    The \emph{average-hard attention function} $\xi$ maps $s$ to the probability distribution $\xi(s)\colon[n]\to[0,1]$ given by
    \begin{equation}\label{eq:average-hard-attention-function}
        \xi(s)_i = 
        \begin{cases}
        1/\mostsim(s) & \text{if } i \in \mostsim(s)\\
        0 & \text{otherwise.}
    \end{cases}
    \end{equation}
\end{definition}

Thus, average-hard attention distributes the entire probability mass evenly among the indices whose values $s_i$ are maximal.

\absatz
The attention head induced by $s$ computes the sequence of $n$ scores, denoted as $\scoring_{x_i}(X)=\scoring_{x_i}(x_1)\cdots\scoring_{x_i}(x_n)$, for each input $x_i$.
Subsequently, this sequence is transformed into a probability distribution using the average-hard attention function $\xi$, and the resulting attention value at position $i$ is the weighted sum of $x_1,\dots,x_n$ based on this distribution.
The formal definition follows.

\begin{definition}[Average-hard attention head]\label{def:head}
    Let $\scoring\colon\F_p^k \times \F_p^k \to \F_p$, be a linear space computable function called a \emph{scoring function}.
    We usually write $\scoring_x(x')$ for $\scoring(x,x')$, called the \emph{score} of $x'$ with respect to $x$.
    
    The \emph{average-hard attention head induced by \scoring} is the function $\head\scoring\colon(\F_p^k)^*\to\F_p^*$, such that for all $n\in\N$, $X=x_1\cdots x_n \in (\F_p^k)^*$, and $i\in[n]$, we have
    \begin{equation}\label{eq:head}
        \head\scoring(X)_i = \trunc{\xi(\scoring_{x_i}(X)) \cdot X\transpose}_p,
    \end{equation}
    where $X\transpose$ is the transpose of $X$ (viewed as an $n$-dimensional vector) and $\cdot$ denotes matrix multiplication.
\end{definition}

While the attention function in a typical transformer model is not average-hard, we specifically focus on the analysis of average-hard attention transformers in this paper.
For clarity, let $s_{ij}$ represent the score assigned to $x_j$ with respect to $x_i$, denoted as $s_{ij}=\scoring_{x_i}(x_j)$.
Applying \cref{def:average-hard-attention-score} to \cref{eq:head} yields
\begin{align}
    \head\scoring(X)_i
    &= \trunc{\sum_{j\in\mostsim(s)} \frac{x_j}{|\mostsim(s)|}}_p \notag\\
    &= \trunc{\frac{1}{|\mostsim(s)|} \cdot \sum_{j \in \mostsim(s)} x_j}_p \label{eqn:att}
\end{align}
for every $i\in[n]$.

\section{Main result}

In this section, we present a construction for attention that, when integrated into the construction of a constant-depth uniform threshold circuit as described by \citet{MerrillSabharwal2023}, enables the complete simulation of an average-hard attention transformer.
Our approach relies on the utilization of a fundamental lemma established by \citet{merrill2023parallelism} (note that this is an earlier version of the same paper).

\begin{lemma}[{\citet[Lemma 3]{merrill2023parallelism}}]\label{lem:2}
    Let $f\colon \{0, 1\}^{*} \to \{0, 1\}$ be a linear space computable boolean function and $c \in \R^{+}$.
    There exists a TM that, for all $n \in \N$, uses $O(\log n)$ space to map input $1^n$ to a circuit of size at most $n^c + c \cdot \log n + 1$ and depth 3 that computes $f$ on inputs of size $c \cdot \log n$.
\end{lemma}

In our paper, we will have functions that transform bit-sequences to bit-sequences and not just to $\{0,1\}$. 
Here, \cref{lem:2} still suffices.
In principle, we could have the input length as an additional input to the circuit, e.g., using a one-hot vector which is 1 at the position of the bit we want to output.
Then we take all these circuits for the entire input length, just copy them and iterate over the position (first one outputs bit 1, second outputs bit 2, ...).
That does not change the size of the circuit significantly, since we have a size of at most $n^c$.
The length is polynomial in $\log n$, so we can produce them all separately.
Hence, \cref{lem:2} can be and is used for functions from now on which output a bit-sequence, e.g., addition.

\absatz
Note that \cref{lem:2} does not mean that $f$ is computable by uniform $\TC^0$ circuits.
The circuits work on input of size $\log n$ and are thus, relative to this, exponentially large.
However, used in a circuit that is applied to a sequence of $n$ values of size $\log n$ each (such as the bit string representation of elements of $\F_p$ the size does indeed become polynomial (now in $n \cdot \log n$)).
This lemma provides a significant implication: certain key operations performed by a transformer head can be effectively implemented using circuits.

\begin{lemma}\label{lem:transformer operations}
Let $p=O(\log n)$.
The functions listed below can be computed by uniform families of  $\TC^0$ circuits, of size polynomial in $n$:
\begin{enumerate}
\item $\scoring\colon\F_p^k \times \F_p^k \to \F_p$: Every scoring function $\scoring$,
\item $\fmax\colon\F_p^n\to\F_p$: computes the maximum of its arguments,
\item $\feq\colon\F_p\times\F_p\to\{0,1\}$: such that $\feq(x,y)=1$ if and only if $x=y$,
\item $\fsel\colon\F_p^k\times\{0,1\}\to\F_p^k$: such that $\fsel(x,y)=x$ if $y=1$ and $\fsel(x,y)=0$ otherwise, for all $x\in\F_p$,
\item\label{sum uniform} $\fsum\colon(\F_p^k)^n\to\F_p^k$: given as the function $\fsum(x_1,\dots,x_n)=\trunc{\sum_{i=1}^nx_i}_p$ for all $x_1,\dots,x_n\in\F_p^k$, and
\item\label{div uniform} $\fdiv\colon\F_p^k\times[n]\to\F_p^k$: given by $\fdiv(x,d)=\trunc{x/d}_p$ for all $x\in\F_p^k$ and $d\in[n]$, where $x / d$ is defined componentwise.
\end{enumerate}
\end{lemma}

\begin{proof}
The computability of scoring functions in linear space follows directly from their definition.
The functions $\fmax$, $\feq$, and $\fsel$ are evidently computable by a TM operating within constant space, as their operations involve simple comparisons and selections.

Statement \ref{sum uniform} corresponds to a reformulation of Lemma~5 by \citet{MerrillSabharwal2023}.
According to this lemma, the function $\fsum$ can be computed by a uniform family of  $\TC^0$ circuits with polynomial size.
Thus, the summation operation can be effectively executed within this computational framework.

We now look at Statement~\ref{div uniform}, first considering integer division and extending this to floating-points.
It is well-established, that for computation of integer division, a TM operating in linear space can perform this operation.
One approach involves left-shifting the second operand by the maximum number of bits, denoted as $k$, such that it does not exceed the value of the first operand.
By subsequently adding $2^k$ to the result and subtracting the bit-shifted second operand from the first operand, the division operation can be iteratively carried out.
Extending this algorithm to floating-point numbers is straightforward, as it primarily involves subtracting the exponents.
Finally, the component-wise extension to $\F_p^k$ is simple.
\end{proof}

\absatz
In the following, we establish the existence of a TM that, given an input of $1^n$, can compute a circuit $C_n$ in logarithmic space.
This circuit, denoted as $C_n$, effectively simulates the operation of a strong average-hard head.

\begin{theorem}\label{thm:main}
The function computed by an average-hard attention head, as defined in \cref{def:head}, can be effectively computed using a uniform family of  $\TC^0$ circuits of polynomial size.
\end{theorem}

\begin{proof}
A schematic representation of the circuit structure for input size $n$, illustrating how it computes the attention vector for the $i$-th input position from $n$ vectors $x_1,\dots,x_n \in \F_p^k$, is presented in \cref{fig:transformer}.
The circuit structure closely adheres to the specifications outlined in \cref{def:head}.

\begin{figure}[!ht]
    \centering
    \includegraphics[width=\columnwidth]{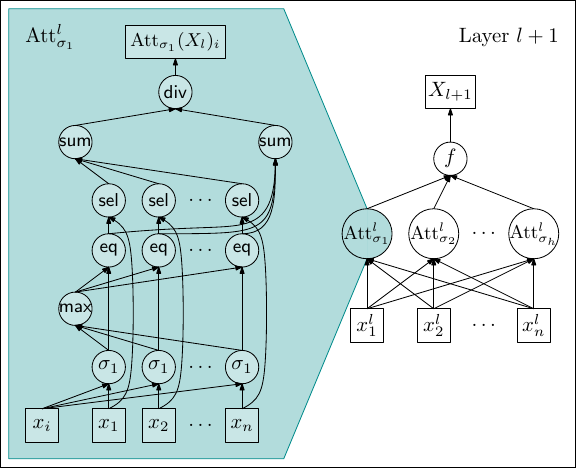}
    \caption{Schematic representation showcasing the structure of a transformer layer (as in \cref{fig:layer}) and one of its average-hard attention heads simulated by a circuit.}
    \label{fig:transformer}
\end{figure}

\absatz
From \cref{lem:transformer operations}, it is evident that all the constituent elements employed in constructing the circuit depicted in \cref{fig:transformer} are uniform families of  $\TC^0$ circuits of polynomial size.
By comparing the various circuit levels with the specifications outlined in \cref{def:head} and utilizing \cref{eqn:att} for the topmost level, it becomes apparent that the circuit accurately computes $\head\scoring(x_1\cdots x_n)_i$.
Moreover, due to the constant depth and polynomial size characteristics of each individual building block, the overall circuit also possesses these required properties.

\absatz
To complete our argument, we need to establish that the circuit depicted in \cref{fig:transformer} can be constructed in logarithmic space by a TM that takes $1^n$ as input.
Given that each of the sub-circuits can be constructed in logarithmic space (as stated in \cref{lem:transformer operations}), our main focus is to demonstrate that the interconnection of the individual sub-circuits, as depicted by the edges in \cref{fig:transformer}, can also be computed within logarithmic space.
Specifically, we aim to show that a fixed number of loops, utilizing loop variables that range between 1 and $n$, are sufficient to generate both the sub-circuits and the edges connecting them.\footnote{It is worth noting that each of the edges shown in \cref{fig:transformer} represents a bundle of $p$ edges, thereby necessitating an additional internal loop variable to generate each of them.}
Since the former is self-evident, we will now focus our attention on the latter aspect.

\absatz
To construct the structure presented in \cref{fig:transformer} for each $i \in [n]$, it is necessary to maintain a variable that tracks the index $i$.

\noindent
To generate the `scores' level and its input edges, an additional loop variable (also ranging from $1$ to $n$) is required to keep track of the index $j \in [n]$ of the sub-circuit being added to the overall circuit, responsible for implementing $\scoring$.
For each $j$, edges are added from both the $j$-th input gate and the $i$-th input gate.
The same approach is employed for the `max' and `select' levels.

\noindent
In the `maximum' level, only one loop variable $j$ is needed to establish edges from each of the $n$ scoring sub-circuits to the single $\fmax$ circuit.
A similar process is followed for the two summation sub-circuits at the `summation' level.
Lastly, there are only two edges that connect to the sub-circuit at the `divide' level.
\end{proof}

\absatz
By incorporating the construction presented in the proof of \cref{thm:main} into the construction of a constant-depth uniform threshold circuit described by \citet{MerrillSabharwal2023}, we achieve a complete simulation of an average-hard attention transformer.
Due to the similarity with the proof provided by \citet{MerrillSabharwal2023}, we outline the proof for brevity.

\begin{theorem}
    Every language that can be decided by a transformer with average-hard attention is in uniform $\TC^0$.
\end{theorem}

\begin{proof}[Proof sketch]
Let $\Sigma = {a_1, \dots, a_m}$ be our alphabet, and let $\omega = a_{i_1}, \ldots, a_{i_n}$ be our input string, where $i_1, \ldots, i_n \in [m]$.
Layer 1 of the transformer receives a positional encoding $\enc(\omega) = \enc(a_{i_1,1}), \ldots, \allowbreak\enc(a_{i_n},n) \in \F_p^k$ as input $X_1$.
Two examples of positional encodings are binary encoding as $\enc(a_{i},j) = (i, j, 0, \ldots, 0)$ and one-hot encoding as $\enc(a_{i},j) = (i, 2^j, 0, \ldots, 0)$.

For each positional encoding $\enc$ (assuming it is log-precision), there exists a $\TC^0$ circuit family that takes the input $w$ (in some binary representation) and produces the output $\enc(w)$.
The existence of such circuits is straightforward for the examples given above, as a logspace-TM can create a circuit that copies the $n$ input symbols to the output and appends the remaining components $j, 0, \ldots, 0$ of the vector as constant outputs.
A counter for $j, 0, \ldots, 0$ is sufficient for this purpose.

The proof proceeds by induction on the number of layers. Since each layer transforms inputs in $(\F_p^k)^n$ to outputs in $(\F_p^k)^n$ by precondition, the induction is trivial.
The main point is to show that using \cref{thm:main}, a single layer can be simulated by a log-space-uniform $\TC^0$ circuit family.

It is crucial to note that all components of a log-precision layer of an average-hard attention transformer are identical to those by \citet{MerrillSabharwal2023}, except for the ones related to average-hard attention.
\citet{MerrillSabharwal2023} demonstrate that each of these components can be simulated by a uniform $\TC^0$ circuit family and can be combined uniformly into one circuit for the entire layer.

By replacing the sub-circuit used for softmax attention in \citet{MerrillSabharwal2023} construction with the circuit from construction \cref{thm:main}, we can obtain circuits for the average-hard attention layer.
\end{proof}

\section{Conclusions and Future Directions}

In conclusion, this paper has shown that log-precision transformers can simulate average-hard attention transformers.
This has significant implications for both theoretical analysis and practical applications of transformer models.

\absatz
Moving forward, there are several promising avenues for future research in this area.
Firstly, an in-depth investigation comparing the expressive power of average-hard and softmax attention transformers would provide valuable insights into the underlying mechanisms of these models.
Understanding whether they possess the same level of expressive capacity or if average-hard attention transformers are strictly less powerful (and to what extent) would shed light on the computational capabilities of transformers.

\absatz
Furthermore, exploring the implications of these findings for practical applications is crucial.
If average-hard attention transformers are found to be equivalent to log-precision transformers, it would provide a more efficient and simplified approach for implementing transformers.
On the other hand, if there are fundamental differences between the two models, it would be important to understand the impact of these differences on the performance and generalization capabilities of transformer-based systems.

\absatz
Addressing the challenge of establishing a comprehensive and concise definition of a transformer that can effectively accommodate various models is crucial for future research in this field.
In the specific context of the compared models in this study, the discrepancies in fundamental definitions posed significant challenges when comparing the models.
This issue extends beyond the scope of this particular paper and is a prevalent obstacle when comparing transformer models in theoretical research.
Therefore, it would be essential to establish a standardized definition that is accessible and convenient for researchers in the field of formal languages to utilize.

\section*{Acknowledgements}
The author would like to acknowledge the valuable feedback provided by Frank Drewes throughout, helpful comments by Gail Weiss, as well as the early-stage discussions with William Merrill, which contributed to the development of this paper.

\bibliographystyle{acl_natbib}
\bibliography{anthology,main}

\end{document}